\documentclass[runningheads]{llncs}
\usepackage{graphicx}
\usepackage{hyperref}
\usepackage{url}

\usepackage{amsthm}
\usepackage{threeparttable}
\usepackage{booktabs}
\usepackage{graphicx}
\usepackage{subfigure}
\usepackage{wrapfig}
\usepackage{amsmath} 
\usepackage{diagbox}
\usepackage{cleveref}
\usepackage[table]{xcolor} 
\usepackage{colortbl} 
\usepackage{amssymb}
\usepackage{bm}
\usepackage{subfigure}
\usepackage{colortbl}
\usepackage{makecell}
\usepackage{threeparttable}
\usepackage{multirow}
\usepackage{caption}
\usepackage{algorithm}
\usepackage{algorithmic}
\usepackage{mathrsfs}
\begin{document}

\title{Obtaining Optimal Spiking Neural Network in Sequence Learning via CRNN-SNN Conversion}
\titlerunning{Obtaining Optimal Spiking Neural Network via CRNN-SNN Conversion}

\author{Jiahao Su\inst{1}\inst{2}\thanks{Work done during the internship at Shanghai Jiao Tong University}
\and Kang You\inst{2} \and Zekai Xu\inst{2} \and Weizhi Xu \inst{1}\and Zhezhi He\inst{2}\thanks{Corresponding author: zhezhi.he@sjtu.edu.cn}}
\authorrunning{J. Su et al.}

\institute{
School of Information Science and Engineering, Shandong Normal University, Jinan, China \\ \email{tssujiahao@gmail.com} \and
School of Electronic Information and
Electrical Engineering, Shanghai Jiao Tong University, Shanghai,
China
}

\maketitle
\begin{abstract}
Spiking neural networks (SNNs) are becoming a promising alternative to conventional artificial neural networks (ANNs) due to their rich neural dynamics and the implementation of energy-efficient neuromorphic chips. 
However, the non-differential binary communication mechanism makes SNN hard to converge to an ANN-level accuracy. When SNN encounters sequence learning, the situation becomes worse due to the difficulties in modeling long-range dependencies. To overcome these difficulties, researchers developed variants of LIF neurons and different surrogate gradients but still failed to obtain good results when the sequence became longer (\textit{e.g.}, $>$500). 
Unlike them, we obtain an optimal SNN in sequence learning by directly mapping parameters from a quantized CRNN. 
We design two sub-pipelines to support the end-to-end conversion of different structures in neural networks, which is called CNN-Morph (CNN $\rightarrow$ QCNN $\rightarrow$ BIFSNN) and RNN-Morph (RNN$ \rightarrow$ QRNN $\rightarrow$ RBIFSNN). Using conversion pipelines and the s-analog encoding method, the conversion error of our framework is zero. Furthermore, we give the theoretical and experimental demonstration of the lossless CRNN-SNN conversion. 
Our results show the effectiveness of our method over short and long timescales tasks compared with the state-of-the-art learning- and conversion-based methods. We reach the highest accuracy of \textbf{99.16\%} (0.46 $\uparrow$) on S-MNIST, \textbf{94.95\%} (3.95 $\uparrow$) on PS-MNIST (sequence length of 784) respectively, and the lowest loss of \textbf{0.057} (0.013 $\downarrow$) within \textbf{8} time-steps in collision avoidance dataset.
\end{abstract}

\section{Introduction} \label{intro}
Spiking Neural Networks (SNNs), known as third-generation neural networks ~\cite{maass1997networks}, are inspired by the biological structure of the brain.
Recent studies have shown that brain-inspired neuron models (\textit{e.g.}, integrate and fire (IF) neuron), can obtain results comparable to ANN networks with high energy efficiency and low latency~\cite{psn,hu2023fast}.
 Unlike traditional ANNs, SNNs use discrete spikes to convey information between neurons. Such binary communication mechanism can be executed smoothly on a neuromorphic chip (\textit{e.g.}, Truenorth\cite{akopyan2015truenorth}, Loihi\cite{davies2018loihi}).

SNNs and RNNs share similarities in many ways, like the design of hidden states and the ability to learn through time. Many efforts have been made in RNN to improve long-term learning dependencies and have achieved astonishing results in sequence learning \cite{IRNN,li2018independently}. Attracted by the performance of RNNs, a question arises:
\emph{how to obtain SNNs that can perform as well as RNNs in sequence learning?}
An obstacle to answering this question is the non-differential binary communication mechanism of SNN, which results in significant information loss. 
To address the problem, surrogate gradient (SG) based back-propagation methods \cite{panda2020toward,neftci2019surrogate} and variants of neurons based learning \cite{huawei2,xu2024bkdsnn} was introduced. However, such approaches still suffer from the spike vanishing phenomenon \cite{panda2020toward} and inaccurate gradient approximation \cite{neftci2019surrogate}. When the temporal sequence becomes longer, SNN cannot achieve the ANN-level accuracy (\textit{e.g.}, SNN-SoTa is 91\% while RNN-SoTa is 97.2\% in permuted-sequential MNIST).

\begin{figure}[t]
\centering
\includegraphics[width=0.9\textwidth]{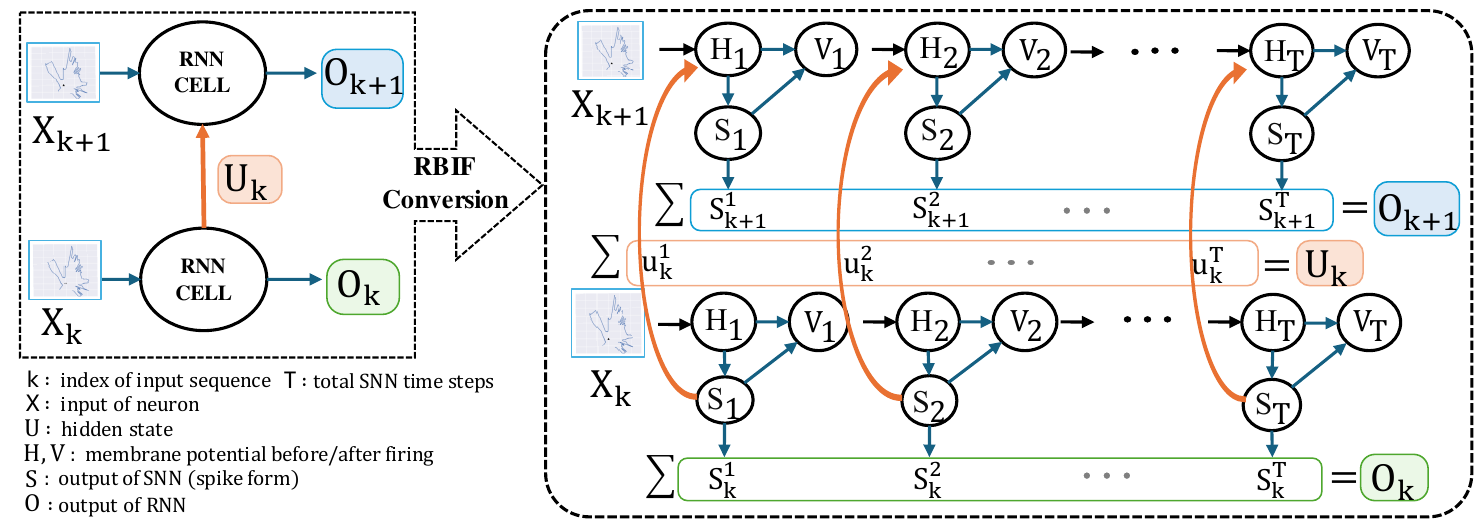}
\caption{\textbf{RNN-RBIF conversion}. A quantized RNN (left) is converted to its corresponding RBIFSNN (right) via QCRC framework without accuracy loss.} 
\label{head}
\end{figure}

Instead of expecting that the non-differential binary network directly converges to ANN-level accuracy through learning, the conversion-based method obtains an SNN by mapping parameters from its counterpart ANN. 
However, existing neuron models in conversion methods \cite{implem,qffs,qcfs} are not compatible with RNN cells because the data in the recurrent structure remain in float type after conversion, which is not allowed. In addition, it still suffers from conversion errors, and these errors will be magnified over time. To address the aforementioned issues, we propose the Recurrent Bipolar Integrate-and-Fire (RBIF) neuron to support the RNN-SNN conversion (as shown in \cref{head}), which guarantees the spike form of the recurrent connection after conversion. Furthermore, we propose a comprehensive framework that supports lossless \textbf{Q}uantized \textbf{C}onvolutional and \textbf{R}ecurrent neural networks to SNN \textbf{C}onversion (QCRC) end-to-end. Our main contributions are summarized as follows:
\begin{itemize}
    \item We propose the Recurrent Bipolar Integrate-and-Fire (RBIF) neuron to address the incompatibility problem of RNN cell. We further give theoretical and experimental proofs of CRNN-SNN conversion. 
    \item We obtain optimal SNN in sequence learning via CRNN-SNN conversion framework, which includes a conversion pipeline with two branches, namely CNN-Morph and RNN-Morph, enabling the conversion of various types of networks into SNNs end-to-end.
    \item 
    We outperform SoTa learning-based works with the accuracy of 99.16\% (0.46 $\uparrow$) on S-MNIST and 94.95\% (3.95 $\uparrow$) on PS-MNIST. We also surpass SoTa conversion-based methods on the collision avoidance dataset, achieving the lowest loss at every time-step (\textit{e.g.}, a loss of 0.118 (0.118 $\downarrow$) at time-step 2, and a loss of 0.057 (0.013 $\downarrow$) at time-step 8). 
\end{itemize}



\section{Related Works}

\subsection{Relation of RNN and SNN}
Spiking neural networks have similarities to vanilla RNN and its variants in the same form. Since the change of membrane potential is related to time, an SNN can be understood as an RNN without recurrent connection \cite{neftci2019surrogate}.
Recurrent neural networks (RNNs) are powerful models for processing sequential data while spiking neural networks (SNNs) show huge potential for processing sequential event-based data.
To address the vanishing and exploding gradient problems during the training of RNN, long short-term memory (LSTM) \cite{hochreiter1997long} is proposed. In addition to adding the gate units in recurrent neurons, other works address the problem by weight initialization like IRNN \cite{IRNN} or changing the form of recurrent neurons like indRNN \cite{li2018independently}. 
Similar to RNNs, many efforts have been made to help SNNs learn long-term patterns. Variants of LIF (\textit{e.g.}, Adaptive LIF \cite{lsnn,huawei}, GLIF \cite{yao2023glif}) are proposed to enlarge the representation of neuronal behaviors. The RSNN that contains recurrent connections is adopted by \cite{huawei2,xing2020new}, resulting in better performance compared with feedforward-only connections. However, it still remains challenges to obtain an SNN with RNN-level performance in the dataset that RNNs are good at, such as sequential image classification and time series forecasting.

\subsection{ANN-to-SNN conversion}
The ANN-to-SNN conversion algorithm was first introduced in \cite{cao2015spiking} by changing the activation function to $ReLU$. \cite{diehl2015fast} presented two  ways to normalize the network weights (\textit{i.e.}, data-based and model-based normalization) to prevent the overestimating output activation. \cite{sengupta2019going,ding2021optimal} took threshold into consideration and proposed different normalization methods. By theoretically analyzing the conversion error between the source ANN and the converted SNN, \cite{qcfs,li2021free} achieved the high-performance ANN-SNN conversion with ultra-low latency. 
To mitigate the sequential error, a neuron that can trigger both positive and negative spikes was proposed, which has been widely used in recent works \cite{hu2023fast,qffs,wang2022signed,you2024ICML}.

The main idea of conversion is to map the firing rate of SNN to the output of quantized ReLU. This idea is in tune with our goal, which is bridging the recurrent dynamics of SNN and RNN. However, previous proofs of ANN-SNN conversion mainly focused on linear and convolutional layers, consequently, the effectiveness of the conversion method was only demonstrated on static datasets. 
In \cref{tab:Technical setting of related works}, we summarize the techniques and settings shared by state-of-the-art works of ANN-SNN conversion. It shows that our work could support different structures of the original ANN and different data types.
\begin{table}[t] 
\caption{\textbf{Technical settings of related works.} ``Eq.'' is the abbreviation of equivalence. ``\checkmark'' represents support. ``m-analog'' indicates the analog input is fed to SNN at every time-step. ``s-analog'' indicates the analog input is only fed to SNN at the first time-step, which is equal to the RNN input. 
}
\centering
\renewcommand\arraystretch{1.2}
\scalebox{1}
{
\resizebox{0.75\textwidth}{16mm}{
            \begin{tabular}{lcccc}
                \hline
                 & QCFS \cite{qcfs} & Offset \cite{offset} & Fast-snn \cite{hu2023fast} & Ours \\ \hline
                Encoding & m-analog & m-analog & m-analog & s-analog\\
                Neuron & IF & IF & signed IF & BIF/RBIF\\
                Theoretical Eq. of CNN  & \checkmark & \checkmark & \checkmark & \checkmark \\
                Theoretical Eq. of RNN & & & & \checkmark \\
                Experimental Eq. of CRNN &  & & & \checkmark \\
                Data Type & static &  static&static &static/temporal \\
                \bottomrule
            \end{tabular}
            }
    }
\label{tab:Technical setting of related works}
\end{table}
\subsection{Quantization in ANN Compression}
Quantization refers to techniques for performing computations and storing tensors at lower bit-widths than floating point precision. The mathematics of quantization for neural networks is as follows:
\begin{equation}
        x_q = clip(round( \frac{x}{s} + z), a, b).
\end{equation}
where $s$ and $z$ denote quantization scale and zero point respectively. $clip(\cdot, a, b)$ function sets the the lower bound $a$ and upper bound $b$. There are two main classes of algorithms: post-training quantization (PTQ) and quantization-aware training (QAT). Compared with PTQ, QAT usually leads to a more robust model. It inserts some fake modules in the computational graph of the model to simulate the effect of the quantization during training, where the straight-through estimator (STE) \cite{STE} is a typical adoption to approximate the gradient of the quantization function. To further mitigate the quantization error, LSQ \cite{lsq} makes $s$ as learnable as other network parameters (\textit{i.e.}, $z, a, b$). We adopt LSQ as our quantization method, following the approach of previous works \cite{bu2022optimal,hu2023fast}.

\section{Method} 
\label{method}
\subsection{SNN Model}

\begin{figure}[t]
\centering
\includegraphics[width=0.95\textwidth]{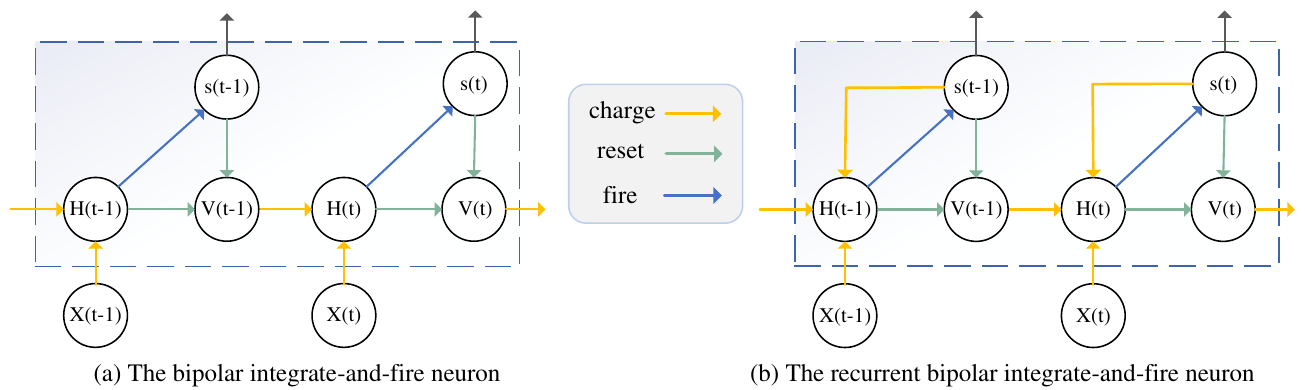}
\caption{The computational graphs of BIF neuron and the RBIF neuron. The recurrent connections in figure (b) are in spike form, which charges $s_{k-1}(t)$ to the H(t) at time-step t. } \label{RBIF}
\end{figure}
\subsubsection{Bipolar Integrate-and-Fire Neuron}
To mitigate the sequential error (the phenomenon
that spikes are generated in spiking neurons where
they should not be), we adopt bipolar integrate-and-fire (BIF) neuron as our basic neuron. (\cref{RBIF} (a))
The overall dynamic of BIF neuron can be expressed as follows:
\begin{align}
    \bm{H}^l(t) &= \bm{V}^l(t-1) + \bm{W}^l\bm{s}^{l-1}(t)\lambda^{l-1}, \label{eq:1} \\
    \bm{V}^l(t) &= \bm{H}^l(t)  - \bm{s}^{l}(t)\lambda^l. \label{eq:2}
\end{align}
where $\bm{H}^l(t)$ and $\bm{V}^l(t)$ represent the membrane potential before and after firing. $\bm{W}^l$ denotes the synaptic weight between layer $l-1$ and layer $l$. To minimize information loss, we adopt the ``reset-by-subtraction" mechanism \cite{rueckauer2017conversion}. Here, $\bm{s}^{l}(t)$ denotes the bipolar output spikes at time step $t$ and $\lambda^l$ represents the threshold of layer $l$. We mitigate the sequential error by allowing $\bm{s}^{l}(t)$ to be either positive or negative while setting a spike tracer $\bm{S}^l(t)$ to record the sum of spikes. The firing rules can be described by the equations below.
\begin{align}
\bm{S}^l(t) &= \bm{S}^l(t-1) + \bm{s}^l(t), \label{eq:3}
\end{align}
where $\bm{S}^l(t)=0, 1,...,S^l_{\textrm{max}}$.
\begin{align}
\bm{s}^l(t) &= 
\begin{cases}
1 ,& \bm{H}^l(t) \geq \lambda^{l} ~ \& ~ \bm{S}^l(t-1)<S^l_{\textrm{max}} \\
0 ,&  \textrm{others} \\
-1 ,&  \bm{H}^l(t) < 0 ~ \& ~ \bm{S}^l(t-1) > 0 
\end{cases}.
\label{eq:4}
\end{align}

\subsubsection{Recurrent Bipolar Integrate-and-Fire Neuron}

As RNN introduces external recurrent connections, the computation graph is different from linear and convolution layers. Accordingly, the pattern of BIF is not compatible with RNN cells, because it will lead to illegal non-spiking forms of recurrent connection after conversion. To address the problem, we propose a novel neuron called the recurrent bipolar integrate and fire (\Cref{RBIF} (b) RBIF) neuron. The neural dynamics of RBIF is defined as:

\begin{figure}[t]
\centering
\includegraphics[width=0.9\textwidth]{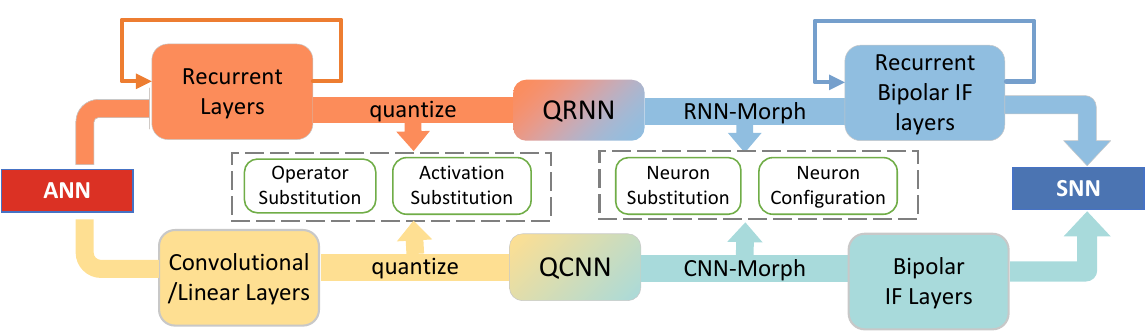}
\caption{\textbf{Conversion pipelines}. Conversion pipeline has two branches, where the top one is RNN-Morph and the bottom one is CNN-Morph. In general, both of the sub-pipelines can be divided into two steps: quantization and Neuron-Morph.} 
\label{fig:pipeline}
\end{figure}
\begin{align}
    \bm{H}_{k}^l(t) &= \bm{V}_{k}^l(t-1) + \bm{W}_{ih}\bm{s}_{k}^{l-1}(t)\lambda^{l-1}+\bm{W}_{hh}\bm{s}_{k-1}^l(t)\lambda^l, \label{eq:5} \\
    \bm{V}_{k}^l(t) &= \bm{H}_{k}^l(t)  - \bm{s}_{k}^{l}(t)\lambda^l. \label{eq:6}
\end{align}

Here we use subscript $k$ to distinguish SNN time-step $t$, which represents the $k$-th input of the RNN sequence. $\bm{W}_{ih}$ and $\bm{W}_{hh}$ denote the learnable input-hidden and hidden-hidden weights respectively. As shown in \cref{RBIF} (b), for the k-th input of a sequence and the l-th layer, the potential of RBIF at time-step $t$ ($\bm{H}(t)$) depends on three parts, the inherited potential $\bm{V}(t-1)$, the charge of the previous layer $\bm{X}^{l-1}$ ($\bm{s}_{k}^{l-1}(t)\lambda^{l-1}$ in \cref{eq:5}), and the output of the $k-1$-th RBIF at the same time-step $\bm{s}_{k-1}^l(t)$. Note that, to avoid sequential errors, we adopt the same firing rules as described in \cref{eq:3,eq:4}.


\subsection{Conversion Pipelines} \label{pipeline}
As illustrated in \cref{fig:pipeline}, QCRC can simultaneously convert different layers to their corresponding SNN layers via two sub-pipelines, which is versatile and suitable for the compound model (\textit{i.e.}, model that contains different types of layers), such as CRNNs.
We design two conversion pipelines for different types of layers in networks, which we call CNN-Morph ($CNN \rightarrow QCNN \rightarrow BIF SNN$) and RNN-Morph ($RNN \rightarrow QRNN \rightarrow RBIF SNN$). In brief, the conversion pipeline can be divided into two steps: the quantization process and the Neuron-Morph process. 

\subsubsection{Quantization.}
\textit{(1) Operator Substitution:} The first step is to make sure all operators in the original ANN are compatible with the SNN. For example, all activation functions should be ReLU based on equivalence requirements before training at full precision. In addition, max-pooling should be replaced by average-pooling because computing maxima with spiking neurons is non-trivial \cite{rueckauer2017conversion}.  
\textit{(2) Activation Substitution:} In this step, the ReLU function is replaced by the quantized ReLU function, where the lower bound $a$ is set to $0$ and upper bound $b$ set to $L$. After the configuration, the quantized ANN is trained using the protocols defined in \cite{lsq,bhalgat2020lsq+}.


\subsubsection{Neuron-Morph.}
\textit{(1) Neuron Substitution:} Benefiting from neuronal equivalence (\cref{Theoretical Equivalence in QCRC}), the synaptic weights of a quantized ANN can be directly mapped to their corresponding SNNs. Specifically, BIF neurons are converted from convolutional/linear neurons, while RBIF neurons are converted from recurrent neurons.   
\textit{(2) Neuron Configuration:} The last step of conversion is to configure the BIF/RBIF neuron attributes (\textit{i.e.}, $\lambda^l, S_{max}^l, V_k^l(0)$) and set the s-analog encoding method for input and bias based on QCRC equivalence requirements. The s-analog encoding is the prerequisite for conversion, that is to make sure the inputs to the $l$ layer of ANN and SNN are the same. Two operations will be performed: a) the current $X$ will be charged into the network only at the first time step, otherwise the input is equal to zero; b) turn off the bias term calculations after the first time step.


\subsection{Theoretical Equivalence in QCRC} \label{Theoretical Equivalence in QCRC}

\begin{table}[t]
\renewcommand\arraystretch{1.2}
\caption{Summary of notations in this paper. 
} \label{tab:notations}
\centering
\scalebox{0.8}
{
\begin{threeparttable}
\begin{tabular}{c|c|c|c} 
\hline
 \textbf{Symbol}  & \textbf{Definition}     & \textbf{Symbol}   & \textbf{Definition}\\ \hline
 $l$  & Layer index      & $\bm{s_{k}^l(t)}$     & Output spike for the k-th input at time-step t \\ 
 $k$  & RNN input index   & $\bm{S_{k}^l(t)}$     & Spike Tracer\tnote{2} ~ at time-step t\\ 
 $t$  & SNN time-step    & $S^l_{\textrm{max}}$ & Maximum value in spike tracer \\
 $\bm{H}_{k}^l(t)$  & Potential before firing   &  $\bm{x_{k}^{l}(t)}$  & UPP\tnote{1} ~ for the k-th input at time-step t\\
 $\bm{V}_{k}^l(t)$  & Potential after firing     &  $\bm{X}_{k}^l(t)$  & UPP Tracer at time-step t   \\
 $T$   & Total time-step   &  $n$ &  Quantization level in ANN\\
 $\lambda^l$      & Trainable threshold in ANN &   $s$ & Learnable quantization scale in ANN\\
  $\bm{W}_{hh}^l$   &   Learnable hidden-hidden weights  & $\lfloor \cdot \rceil$        & Round operation  \\
 $\bm{W}_{ih}^l$    & Learnable input-hidden weights & $clip(x, a, b)$      & Clip function that limits x between a and b  \\  \hline
\end{tabular}
        \begin{tablenotes}
		\footnotesize
		\item[1] Unweighted postsynaptic potential 
            \item[2] Tracer records the sum of the first t values.
	\end{tablenotes}
\end{threeparttable}
}
\end{table}
\label{zhengming}

\begin{theorem} \label{throrem1}
Assume a quantized CNN with ReLU activation function parameterized by $\bm{W^l}$ is converted to a BIFSNN based on CNN-Morph and s-analog encoding is adopted, then the accumulated outputs of the SNN are equal to the quantized CNN outputs when T is long enough that remaining membrane potential is insufficient to fire a spike.
\end{theorem}
\begin{proof}
\renewcommand{\qedsymbol}{}
\Cref{throrem1} proof is in the appendix.
\end{proof}

\begin{theorem} \label{throrem2}
Suppose an RNN with ReLU activation function, parameterized by $\bm{W}_{ih}$ and $\bm{W}_{hh}$, is quantized into $n$ quantization level by quantization scale $s$:  
\begin{equation}
    \bm{h}_{k} = s \cdot clip(\lfloor \dfrac{\bm{W}_{ih}\bm{x}_{k}+b_{ih}+\bm{W}_{hh}\bm{h}_{k-1}+b_{hh}}{s} \rceil,0,n). \label{qrnn}
\end{equation}
If an RBIFSNN is converted from the QRNN with $\bm{V}_{k}^{l}(0) = 0.5s$, $\bm{S}_{max}^{l} = n$, $\lambda^l = s$ and the s-analog encoding is adopted, then for any $k$-th input of the RNN sequence, the accumulated outputs of the SNN is equal to the QRNN output:
\begin{equation}
    \bm{X}_{k}^{l}(T) = \bm{h}_{k}, \label{conclusion}
\end{equation}
when T is long enough that the remaining membrane potential is not sufficient to fire a spike.
\end{theorem}

\begin{proof}
\renewcommand{\qedsymbol}{}
The key idea of QRNN-RBIF conversion is that for each RNN sequence input, the activation value of the RNN neuron can be equivalently mapped to the accumulated output of the SNN neuron. Based on this, we first combine \cref{eq:5} and \cref{eq:6} to get the potential update equation:

\begin{equation}
    \bm{V}_{k}^l(t) - \bm{V}_{k}^l(t-1) =\bm{W}_{ih}\bm{s}_{k}^{l-1}(t)\lambda^{l-1}+\bm{W}_{hh}\bm{s}_{k-1}^l(t)\lambda^l - \bm{s}_{k}^{l}(t)\lambda^l. \label{eq:7}
\end{equation}
By summing up \cref{eq:7} from 1 to inference time-step $T$, we have:
\begin{equation}
\label{eq:8}
 \bm{V}_{k}^l(T) - \bm{V}_{k}^l(0) = \bm{W}_{ih}\lambda^{l-1}\sum_{i=1}^{T} \bm{s}_{k}^{l-1}(t)+\bm{W}_{hh}\lambda^l\sum_{i=1}^{T}\bm{s}_{k-1}^l(t) - \lambda^l\sum_{i=1}^{T}\bm{s}_{k}^{l}(t),
\end{equation}
where $\sum_{i=1}^{T}\bm{s}_{k}^{l}(t) = \sum_{i=1}^{T}(\bm{S}_{k}^{l}(t)-\bm{S}_{k}^{l}(t-1)) = \bm{S}_{k}^{l}(T) - \bm{S}_{k}^{l}(0)$ according to \cref{eq:3}. If we set $\bm{S}_{k}^{l}(0)$ = 0, \cref{eq:8} can be simplified as
\begin{equation}
\label{eq:9}
 \bm{V}_{k}^l(T) - \bm{V}_{k}^l(0) = \bm{W}_{ih}\lambda^{l-1}\bm{S}_{k}^{l-1}(T)+\bm{W}_{hh}\lambda^l\bm{S}_{k-1}^l(T) - \lambda^l\bm{S}_{k}^{l}(T).
\end{equation}
Then, we divide both sides of \cref{eq:9} by the threshold $\lambda^{l}$. With additional simple transformation, we can obtain the expression for spike tracer:
\begin{equation}
    \label{eq:10}
    \bm{S}_{k}^{l}(T) =\dfrac{(\bm{W}_{ih}\lambda^{l - 1}\bm{S}_{k}^{l-1}(T)+\bm{W}_{hh}\lambda^l\bm{S}_{k-1}^l(T)+\bm{V}_{k}^l(0)-\bm{V}_{k}^l(T) )}{\lambda^{l}}.   
\end{equation}
When the simulation time-steps $T$ is long enough so that the remaining membrane potential $\bm{V}_{k}^l(T)$ is insufficient to fire a spike, the \cref{eq:10} can be written as
\begin{equation}
    \label{eq:11}
    \bm{S}_{k}^{l}(T) =\left\lfloor \frac{\bm{W}_{ih}\lambda^{l - 1}\bm{S}_{k}^{l-1}(T)+\bm{W}_{hh}\lambda^l\bm{S}_{k-1}^l(T)+\bm{V}_{k}^l(0)}{\lambda^{l}} \right\rfloor,   
\end{equation}
where $\bm{S}_{k}^l(T)=0, 1,...,S^l_{\textrm{max}}$. By multiplying both sides of the \cref{eq:11} by $\lambda^{l}$ and inserting the clip function, we can get the final equation:
\begin{equation}
    \label{eq:12}
    \bm{X}_{k}^{l}(T) =\lambda^{l}\cdot clip (\lfloor \dfrac{\bm{W}_{ih}\bm{X}_{k}^{l-1}(T)+\bm{W}_{hh}\bm{X}_{k-1}^l(T)+\bm{V}_{k}^l(0)}{\lambda^{l}} \rfloor, 0, \bm{S}^l_{max}),   
\end{equation}
where $\bm{X}_{k}^{l}(T)=\lambda^{l}\bm{S}_{k}^{l}(T)$ by definition.

\Cref{eq:12} describes the relationship between unweighted postsynaptic potential of RBIF neurons in adjacent layers.
By setting $\lambda^{l} = s$, $\bm{S}_{max}^{l} = n$, $  \bm{V}_{k}^{l}(0) = 0.5s + b_{ih} + b_{hh}$, \cref{eq:12} and \cref{qrnn}
are equivalent, which will lead to the conclusion in \cref{conclusion}. Note that, setting $\bm{V}_{k}^{l}(0) = 0.5s$, which is called pre-charge method in \cite{bu2022optimal}, will make operator $\lfloor \cdot \rfloor$ and operator $\lfloor \cdot \rceil$ equal.
\end{proof}

\section{Experiments}
In this section, we obtain optimal SNNs in sequence learning via CRNN-to-SNN conversion. We validate the effectiveness of our method with other state-of-the-art learning-based approaches and conversion-based approaches, demonstrating the advantages of our method on different datasets (\textit{i.e.}, benchmark S-MNIST/pS-MNIST \cite{mnist} and collision avoidance dataset \cite{icra}).
We further experimentally demonstrate the lossless conversion of QCRC and the effectiveness of s-analog encoding in the ablation study. 

\subsection{Implementation details}
The experiments exactly follow the quantization and conversion stages as introduced in \cref{pipeline}. Both ANN quantization
training and SNN implementation are carried out with PyTorch. Unless otherwise specified, the optimizer is Adam, the learning rate scheduler is the cosine annealing schedule.
\subsubsection{S-MNIST and pS-MNIST}
 We only apply normalization transform to the dataset. The main hyper-parameters of the models follow their corresponding papers \cite{lsnn,huawei,psn}. Training epoch and batch size are 200 and 256 for all models. The learning rate of our model is 0.0002. The cross-entropy loss (CE) is used to evaluate the difference between the estimated value and the actual value.
\subsubsection{Obstacle detection and avoidance}
The total dataset (including 20 training, 5 validation traces) is split into multiple sub-sequences of length 32 and fed into the model sequentially. The input of the LIDAR scanner will be fed to the main structure, while the estimated robot pose will be firstly clipped to the range of -1.0 to 1.0 and then concatenated with the output of layer 5 before sent to the next layer. We follow a similar network as \cite{icra}, consisting of an RNN preceded by a set of convolutional layers. The epoch and batch size are set to 1000 and 32 respectively. We use a fixed learning rate of 0.0001 to train the model with the mean square error (MSE) loss function.

\subsection{Sequential MNIST}

The sequential- and permuted-sequential MNIST (S/PS-MNIST) \cite{mnist} are widely used benchmarks to verify the learning ability for long-term dependencies. The image will be divided into 784 pixels and sent to the network pixel by pixel. The networks are asked to predict the class of MNIST image only when all 784 pixels are fed sequentially to the recurrent network. Therefore, achieving high accuracy on the ``pixel-by-pixel MNIST" problem is not easy because neurons must have the ability to learn from the long contexts.

Benefiting from the high scalability of our method, we use indRNN cell \cite{li2018independently} as the original RNN model. We set the quantization step and time-steps to 128 and 512. A performance comparison is given in \Cref{tab:S-MNISTsota}. RBIF reads the image pixel by pixel without any extra encoding process, just as the same as the LSTM. It outperforms all models, achieving 99.16\% and 94.95\% classification accuracy on S-MNIST and PS-MNIST respectively. Note that, the accuracy of pr-ALIF (94.3\%) on PS-MNIST is not included for comparison because the adoption of a sliding window is unfair to other models.
We also compare our method with the conversion-based method. It turns out that performance deteriorates rapidly as the sequence gets longer due to the propagation of sequential error, which we will explain further in \cref{effectofananlog}.
\begin{table}[t]
\centering 
\caption{Test accuracy (\%) of S-MNIST and PS-MNIST. - refers to data not reported or cannot be reproduced. Current and compared best results are in \textbf{bold} and \colorbox{gray!20}{grey} respectively.
}
\begin{threeparttable}
    \begin{tabular}{ccccccccc}
         \toprule
         \diagbox{Dataset}{Neurons}~  
         &~ RBIF ~
     & ~\begin{tabular}[c]{@{}c@{}}LSTM \\ \cite{hochreiter1997long} \end{tabular}  
     & ~\begin{tabular}[c]{@{}c@{}}{pr-ALIF} \\ \small{\cite{huawei2}} \end{tabular}  
     & ~\begin{tabular}[c]{@{}c@{}}ALIF\\ \cite{huawei} \end{tabular}  
     & ~\begin{tabular}[c]{@{}c@{}}MPSN\\ \cite{psn} \end{tabular}  
     & ~\begin{tabular}[c]{@{}c@{}}LSNN\\ \cite{lsnn} \end{tabular}  
         &~LIF ~
         \\ \midrule                  
\begin{tabular}[c]{@{}c@{}} S-MNIST \end{tabular} & \textbf{99.16} & 98.2& \cellcolor{gray!20} 98.7 &  97.82 &63.6 & 96.4  &  28.6
         \\
 PS-MNIST  &   \textbf{94.95}  &88 & 94.3 \tnote{1}  & \cellcolor{gray!20} 91 &34.9 & -  &  23.9
         \\
         \bottomrule
    \end{tabular}
    \begin{tablenotes}
    \item[1] not included for comparison due to the adoption of a sliding window. 
    \end{tablenotes}
    \label{tab:S-MNISTsota}
\end{threeparttable}
\end{table}

\subsection{Obstacle detection and avoidance} \label{obstacle}
To explore the application of SNNs in sequential robotic tasks, we conduct robot navigation experiments using the dataset proposed in \cite{icra}.
The objective of this task is to navigate a Pioneer 3-AT mobile robot safely through obstacles. Specifically, the network input comprises data streams from a 270-degree 2D LiDAR scanner and a time series of estimated robot poses sampled at 10Hz. By generating a decision in the form of a target angular velocity, the network can maneuver the robot safely around the obstacles.


\Cref{tab:collisionavodiance} reports the results of the collision avoidance dataset, our method outperforms the others at every time-step with a lower MSE loss. Note that, because of the incompatibility problem mentioned in \cref{intro}, we add a recurrent structure to neurons in compared works, which is compatible to their conversion algorithm. It clearly shows that the IF neuron still suffers a lot from the conversion errors in QCFS. The calibrating offset spikes in \cite{offset} can bridge the gap between ANN and SNN in a certain way, but cannot control the errors from the source. Fast-SNN can achieve nearly lossless performance but cannot eliminate the sequential errors in the end. In contrast, QCRC achieves the lowest loss of 0.0569 when $T=8$, which is equal to the loss of ANN. Even when the time-steps is 2, we can achieve a very low loss of 0.1180.

\begin{table}[b]
    \caption{Experiment results on collision avoidance dataset. $\sigma$ denotes the time-steps to calculate offset spikes. - means no result can be obtained under L = 8 according to \cite{hu2023fast}. Best results are in \textbf{bold}. 
    }
    \label{tab:collisionavodiance}
\centering
\renewcommand\arraystretch{1.2}
\scalebox{1}
{
\begin{tabular}{cccccccc}\hline
Method & Neuron               & ANN & T=2  & T=4  & T=8  & T=16  & T=32 \\ \hline
QCFS \cite{qcfs} & IF & 0.0907 & 0.2726 & 0.1929 & 0.1469 & 0.1208 & 0.1078 \\
Offset ($\sigma = 6$) \cite{offset} & IF & 0.0907 & 0.2187 & 0.1414 & 0.1125 & 0.0985 &
0.0969 \\
Fast-SNN \cite{hu2023fast} & signed IF & 0.0669 & 0.2364 & 0.1356 & 0.0694 & -     & -    \\
\textbf{Ours} & BIF/RBIF & 0.0569 & \textbf{0.1180} & \textbf{0.0780} & \textbf{0.0569} &  0.0569 & 0.0569
\\ \hline
\end{tabular}
}
\end{table}

\begin{figure}[t]
\centering
\includegraphics[width=0.9\textwidth]{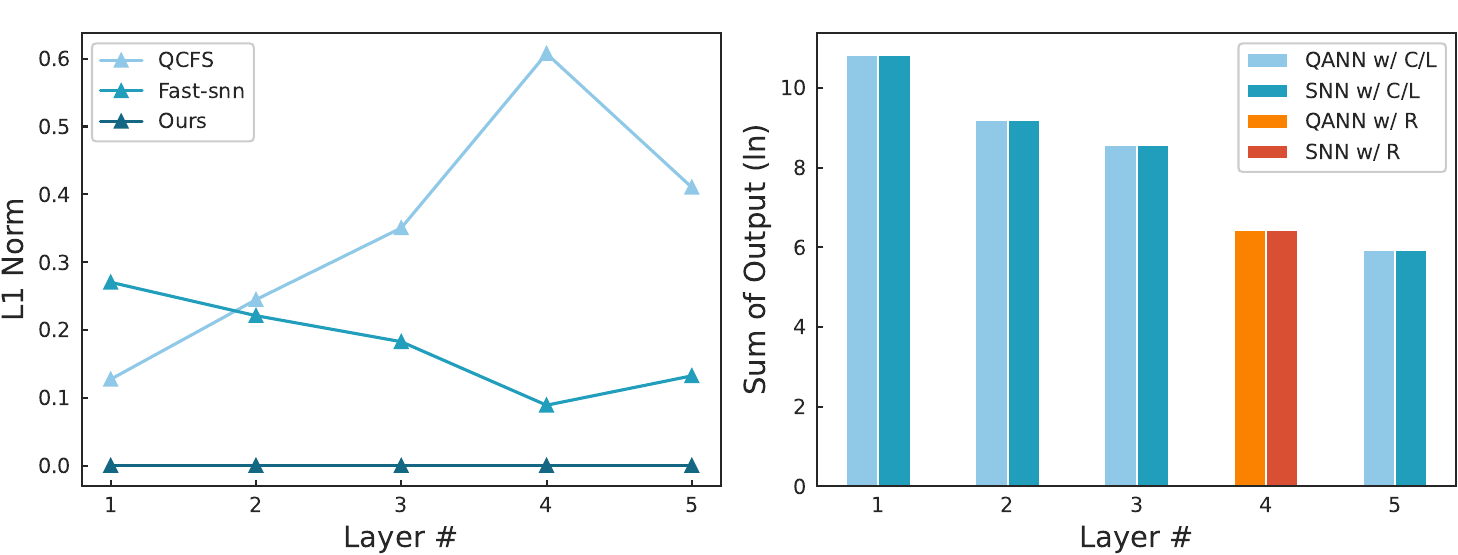}
\caption{\textbf{Conversion error study.} (left) The L1 Norm between the feature map of accumulated spiking output and the feature map of quantized activation output. (right) The sum output of the activation layers in QCRC. ``C/L" refers to convolutional/linear layers and ``R" denotes recurrent layers.}.
\label{last}
\end{figure}
\subsection{Ablation Study}
\subsubsection{Conversion Error Analysis}
    We perform two-fold validation on the equivalence of QCRC. We use the dataset and network in \cref{obstacle}. The choice of CRNN network can make the analysis more comprehensive since it contains three commonly used layers (\textit{i.e.}, linear, convolutional, recurrent layers). To measure the conversion error straightforwardly, we use a batch of data to visualize the L1 Norm (a.k.a. Manhattan distance) between QANN and its counterpart SNN for intermediate activation layers, as shown in the left of \cref{last}. It is shown that the use of IF neurons makes the L1 Norm in QCFS remain at a large value due to the accumulating sequential error. 
Although Fast-snn proposes the signed IF neuron and layer-wise fine-tuning scheme to mitigate the sequential error, the m-analog encoding still leads to in-equivalence at the model level and degrades the performance at deeper layers.
Compared with them, only QCRC reaches the true lossless level (\textit{i.e.}, the L1 norm between QANN and converted SNN is 0). Furthermore, the bar graphs we draw (right part in \Cref{last}) show that the sum of activations for each neuron layer in QANN and SNN is equal.

\begin{figure}[t]
\centering
\includegraphics[width=0.95\textwidth]{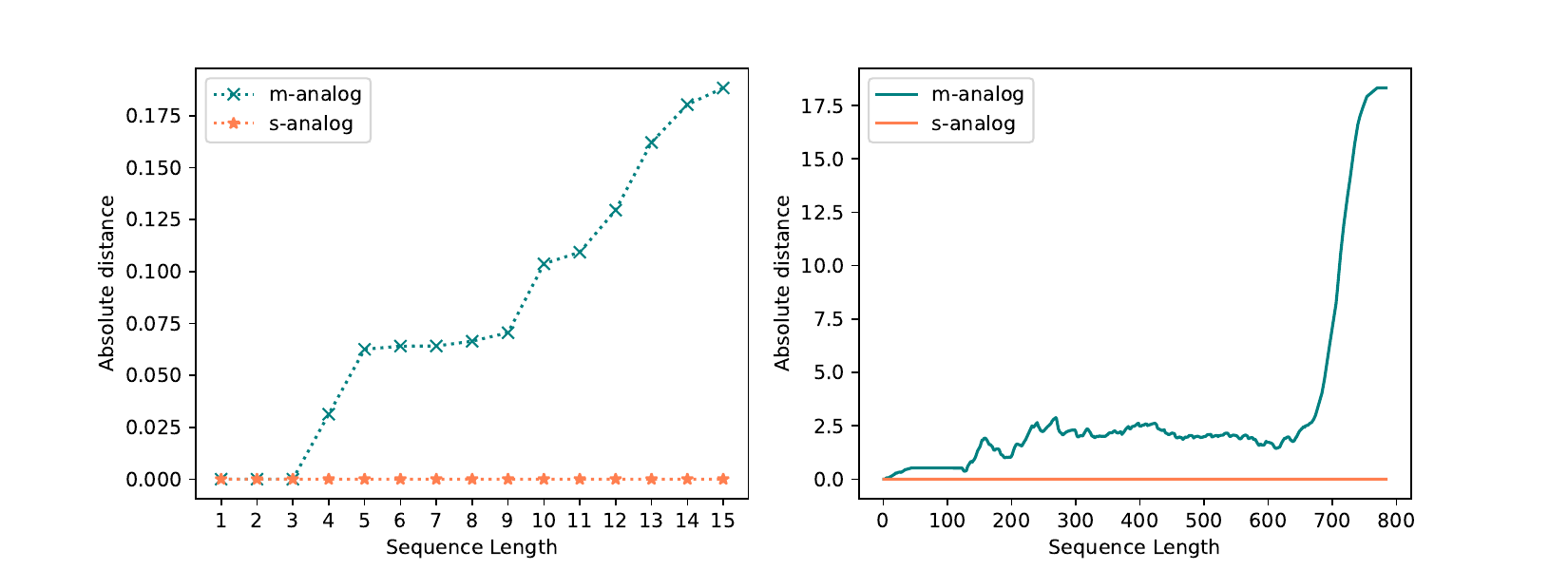}
\caption{The absolute distance between quantized ANN and SNN for two analog encoding methods. The results show the first 15 time steps in the sequence (left) and the entire sequence (right). }.
\label{fig:analog}
\end{figure}
\subsubsection{Effect of s-analog Encoding} \label{effectofananlog}
We perform experiments on sequential MNIST to explore the effects of different analog encoding methods. The m-analog encoding (used in \cite{qcfs,hu2023fast}) charges the current $X$ into the network at every time-step, while the s-analog encoding we use only charges X into the network at the first time step. 
We randomly select an image in MNIST and evaluate the two analog encodings using the same network with the same settings.
\Cref{fig:analog} depicts the results of the absolute distance between quantized ANN and SNN for the first 15 time steps in the sequence (left) and the entire sequence (right). We can see that the m-analog encoding only mitigates the sequential error. If the sequential error is not zero for one input (\textit{e.g.}, the fourth input of the sequence), it will propagate pixel by pixel by the recurrent structure and will be magnified by larger quantization error. When the sequence gets longer, the accumulated error will rapidly grow and become uncontrollable. Since we only need the output of the last time step, accumulated error will severely degrade the performance. 
In contrast, the adoption of s-analog encoding together with our conversion pipelines guarantees lossless conversion at every time-step in the sequence.

\section{Discussion and Conclusion}
This paper proposes a comprehensive QCRC framework to help SNNs overcome the challenge of not achieving ANN-level results in sequence learning, enabling SNNs to achieve results comparable to RNNs.  
To overcome the incompatibility problem of RNN cell, we propose RBIF neuron. Based on this, we further demonstrate the lossless CRNN-SNN conversion with the design of conversion pipelines and s-analog encoding.
The framework includes two sub-pipelines (\textit{i.e.}, CNN-Morph and RNN-Morph), which can support end-to-end conversion of complex models with both recurrent and convolutional structures into SNN and is not limited by the type of dataset. 
We are the first work to implement lossless RNN-SNN conversion on time series tasks. Our results show promising advantages compared to the state-of-the-art conversion- and learning-based methods. Our results answer the question in \cref{intro}: \emph{we can easily achieve ANN-level performance for SNNs in sequence learning via CRNN-SNN conversion.} We believe our work paves the way for the application of SNNs in time series tasks.

\subsubsection{Acknowledgment.}
This work is partially supported by National Key R\&D Program of China (2022YFB4500200), National Natural Science Foundation of China (Nos.62102257), Biren Technology–Shanghai Jiao Tong University Joint Laboratory Open Research Fund, Microsoft Research Asia Gift Fund, Shandong Normal University Undergraduate Research Fund.

\appendix
\section*{Appendix}
\setcounter{theorem}{0}
\begin{theorem} \label{throrem1}
Assume a quantized CNN with ReLU activation function parameterized by $\bm{W^l}$ is converted to a BIFSNN based on CNN-Morph and s-analog encoding is adopted, then the accumulated outputs of the SNN is equal to the quantized CNN output when T is long enough that remaining membrane potential is insufficient to fire a spike.
\end{theorem}

\begin{proof}
\renewcommand{\qedsymbol}{}
We first combine \cref{eq:1} and \cref{eq:2} to get the potential update equation:
\begin{equation}
    \bm{V}^l(t) - \bm{V}^l(t-1) = \bm{W}^l\bm{s}^{l-1}(t)\lambda^{l-1} - \bm{s}^{l}(t)\lambda^l. 
    \label{16}
\end{equation}
By summing up \cref{16} from 1 to inference time-step $T$, we have:
\begin{equation}
 \bm{V}^l(T) - \bm{V}^l(0) = \bm{W}^l\lambda^{l-1}\sum_{i=1}^{T} \bm{s}^{l-1}(t) - \lambda^l\sum_{i=1}^{T}\bm{s}^{l}(t).
 \label{17}
\end{equation}
where $\sum_{i=1}^{T}\bm{s}^{l}(t) = \sum_{i=1}^{T}(\bm{S}^{l}(t)-\bm{S}^{l}(t-1)) = \bm{S}^{l}(T) - \bm{S}^{l}(0)$ according to \cref{eq:3}. If we set $\bm{S}^{l}(0)$ = 0, \cref{17} can be simplified as:
\begin{equation}
 \bm{V}^l(T) - \bm{V}^l(0) = \bm{W}^l\lambda^{l-1}\bm{S}^{l-1}(T) - \lambda^l\bm{S}^{l}(T).
 \label{18}
\end{equation}
Then, we divide both sides of \cref{18} by the threshold $\lambda^{l}$. With additional simple transformation, we can obtain the expression for spike tracer:
\begin{equation}
    \bm{S}^{l}(T) =\dfrac{\bm{W}^l\lambda^{l - 1}\bm{S}^{l-1}(T)+\bm{V}^l(0)-\bm{V}^l(T)}{\lambda^{l}}   
    \label{20}
\end{equation}
When the simulation time-steps $T$ is long enough so that the remaining membrane potential $\bm{V}^l(T)$ is insufficient to fire a spike, \cref{20} can be rewritten as the expression of :
\begin{equation}
    \bm{S}^{l}(T) =\left\lfloor \frac{\bm{W}^l\lambda^{l - 1}\bm{S}^{l-1}(T)+\bm{V}^l(0)}{\lambda^{l}} \right\rfloor,
    \label{21}
\end{equation}
where $\bm{S}^l(T)=0, 1,...,S^l_{\textrm{max}}$. By multiplying both sides of the \cref{21} by $\lambda^{l}$, we can get the final equation:
\begin{equation}
    \bm{X}^{l}(T) =\lambda^{l}\cdot clip (\lfloor \dfrac{\bm{W}^l\bm{X}^{l-1}(T)+\bm{V}^l(0)}{\lambda^{l}} \rfloor, 0, \bm{S}^l_{max}), 
    \label{22}
\end{equation}
where $\bm{X}^{l}(T)=\lambda^{l}\bm{S}^{l}(T).$ by definition. \par
\Cref{22} describes the relationship between unweighted postsynaptic potential of BIF neurons in adjacent layers.

Considering a quantization CNN with quantization scale $s$ and quantization level $n$:
\begin{equation}
    \bm{X'} = s \cdot clip( \lfloor \dfrac{\bm{W}^l\bm{X}^{l-1} + b}{s} \rceil, 0, n).
    \label{23}
\end{equation}
If we set $\lambda^l = s$, $\bm{S}_{max}^{l} = n$, $\bm{V}^l(0) = b + 0.5s$, \cref{23} and \cref{22} are equivalent.

\end{proof}
\bibliography{reference}
\bibliographystyle{splncs04}
\newpage

\end{document}